\documentclass{article} 
\usepackage{arxiv,times}

\usepackage{amsmath,amsfonts,bm}

\def\eqref#1{equation~\ref{#1}}

\def\1{\bm{1}}

\DeclareMathAlphabet{\mathsfit}{\encodingdefault}{\sfdefault}{m}{sl}
\SetMathAlphabet{\mathsfit}{bold}{\encodingdefault}{\sfdefault}{bx}{n}

\newcommand{\E}{\mathbb{E}}

\usepackage{hyperref}
\usepackage{url}

\usepackage{times}  
\usepackage{helvet}  
\usepackage{courier}  
\usepackage{graphicx} 
\usepackage{natbib}  
\usepackage{caption} 

\usepackage{amsthm}
\usepackage{graphicx} 
\usepackage{amsmath}
\usepackage{amssymb}
\newtheorem{theorem}{Theorem}

\newtheorem{proposition}[theorem]{Proposition}
\newtheorem{Definition}[theorem]{Definition}
\newtheorem{Theorem}[theorem]{Theorem}

\usepackage{blindtext}
\usepackage{textcomp}
\usepackage{algorithm}
\usepackage[noend]{algpseudocode}
\usepackage{textcomp}
\usepackage{graphicx}
\usepackage{booktabs} 
\usepackage{subfig}

\usepackage{newfloat}
\usepackage{listings}
\usepackage[T1]{fontenc}
\usepackage[utf8]{inputenc}
\usepackage{authblk}  
\usepackage{graphicx}
\usepackage{amssymb}
\usepackage{lineno}
\usepackage{multirow}
\usepackage{booktabs}

\title{Bridging the Fairness Divide: Achieving Group and Individual Fairness in Graph Neural Networks}

\author[1]{Duna Zhan}
\author[2]{Dongliang Guo}
\author[1]{Pengsheng Ji}
\author[2]{Sheng Li}
\affil[1]{University of Georgia}
\affil[2]{University of Virginia}
\affil[ ]{{dunaz@uga.edu, dongliang.guo@virginia.edu, psji@uga.edu, shengli@virginia.edu}}

\begin{document}
\maketitle

\begin{abstract}
Graph neural networks (GNNs) have emerged as a powerful tool for analyzing and learning from complex data structured as graphs, demonstrating remarkable effectiveness in various applications, such as social network analysis, recommendation systems, and drug discovery. However, despite their impressive performance, the fairness problem has increasingly gained attention as a crucial aspect to consider. Existing research in graph learning focuses on either group fairness or individual fairness. However, since each concept provides unique insights into fairness from distinct perspectives, integrating them into a fair graph neural network system is crucial. To the best of our knowledge, no study has yet to comprehensively tackle both individual and group fairness simultaneously. In this paper, we propose a new concept of individual fairness within groups and a novel framework named Fairness for Group and Individual (FairGI), which considers both group fairness and individual fairness within groups in the context of graph learning. FairGI employs the similarity matrix of individuals to achieve individual fairness within groups, while leveraging adversarial learning to address group fairness in terms of both Equal Opportunity and Statistical Parity. The experimental results demonstrate that our approach not only outperforms other state-of-the-art models in terms of group fairness and individual fairness within groups, but also exhibits excellent performance in population-level individual fairness, while maintaining comparable prediction accuracy.
\end{abstract}

\section{Introduction}
\label{sec:intro}

Graph-based data provides a natural way to present complex relationships and structures in the real world and has wide applications in various domains, such as social networks and recommendations. Graph Neural Networks (GNNs) have emerged as powerful tools for graph-structured data, including Graph Convolutional Networks (GCNs) \citep{kipf2016semi}, Graph Attention Neural Networks (GAT) \citep{wang2019kgat}, and Graphsage \citep{hamilton2017inductive}. Despite the impressive performance of these models, a notable limitation is that GNNs can potentially be biased and exhibit unfair prediction when the training graph is biased or contains sensitive information. 

Existing work on fair graph learning mainly focuses on group, individual, and counterfactual fairness. Models emphasizing group fairness concentrate on mitigating bias at the demographic group level and guaranteeing fairness for protected groups such as FairGNN \citep{dai2021say}.  Models prioritizing individual fairness aim to ensure fairness at the individual level, such as InFoRM \citep{kang2020inform}. Graph learning with counterfactual fairness assesses fairness by assuring the fairness of predictions for each individual compared to counterfactual scenarios, like GEAR \citep{ma2022learning}.

When focusing on a single type of fairness, like individual or group fairness, the existing fair graph learning models demonstrate effectiveness in mitigating bias while achieving comparable accuracy. However, group and individual fairness possess inherent limitations, and integrating them is not a trivial task. Group fairness measurements such as Statistical Parity (SP) \citep{dwork2012fairness} and Equal Opportunity (EO) \citep{hardt2016equality} only consider fairness at the demographic level, neglecting individual-level fairness. Conversely, individual fairness prioritizes equity on a personal scale but falls short of ensuring fairness across broader demographic groups.

For instance, Fig. \ref{fig1} presents an example of an admissions model where gender is identified as the sensitive attribute: females are represented by red, and males by blue. In Fig. \ref{fig1}(a), the machine learning model ensures group fairness, signifying that it upholds fairness at the demographic level, but overlooks individual fairness within these groups. In contrast, the model illustrated in Fig. \ref{fig1}(b) maintains fairness at both the demographic and individual levels. It adheres to the principle that similar inputs should yield similar outputs, a criterion essential for individual fairness, which is achieved through the Lipschitz condition. Consequently, this machine learning model not only ensures fairness across groups but also assures that candidates with analogous attributes, like grades and experience, receive comparable outcomes. Such consistency is crucial for preserving equity throughout the admissions process. 
To the best of our knowledge, this paper is the first study that simultaneously achieves both group and individual fairness within groups. Our work addresses and mitigates unfairness at both the group and individual levels, as exemplified in Fig. \ref{fig1}(b).
\begin{figure}[t]
    \centering
    \vspace{-5mm}
    \subfloat[\centering Machine learning model with group fairness]{{\includegraphics[scale=0.2]{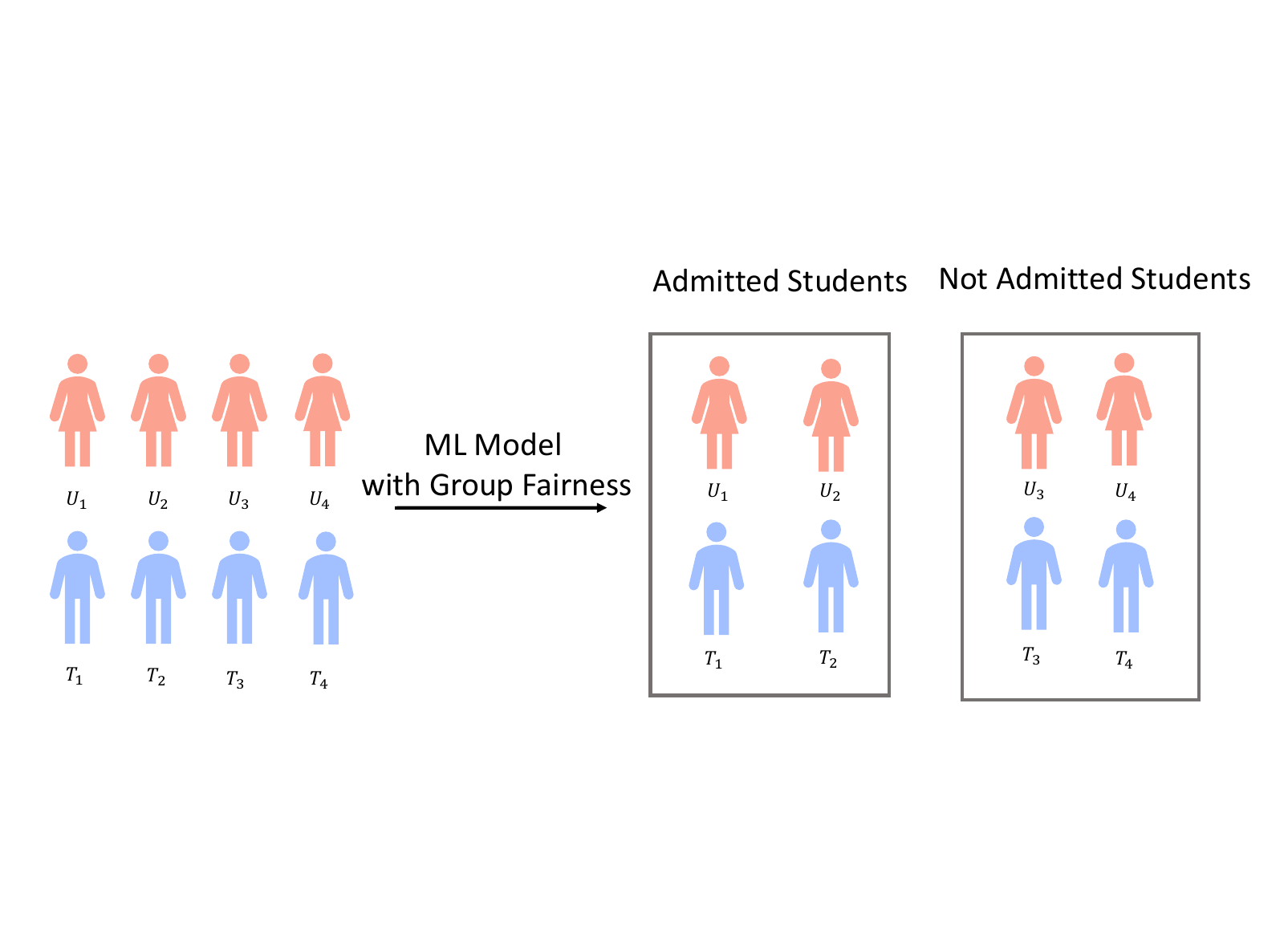} }}%
    \qquad
    \subfloat[\centering Machine learning model with group fairness and individual fairness]{{\includegraphics[scale=0.18]{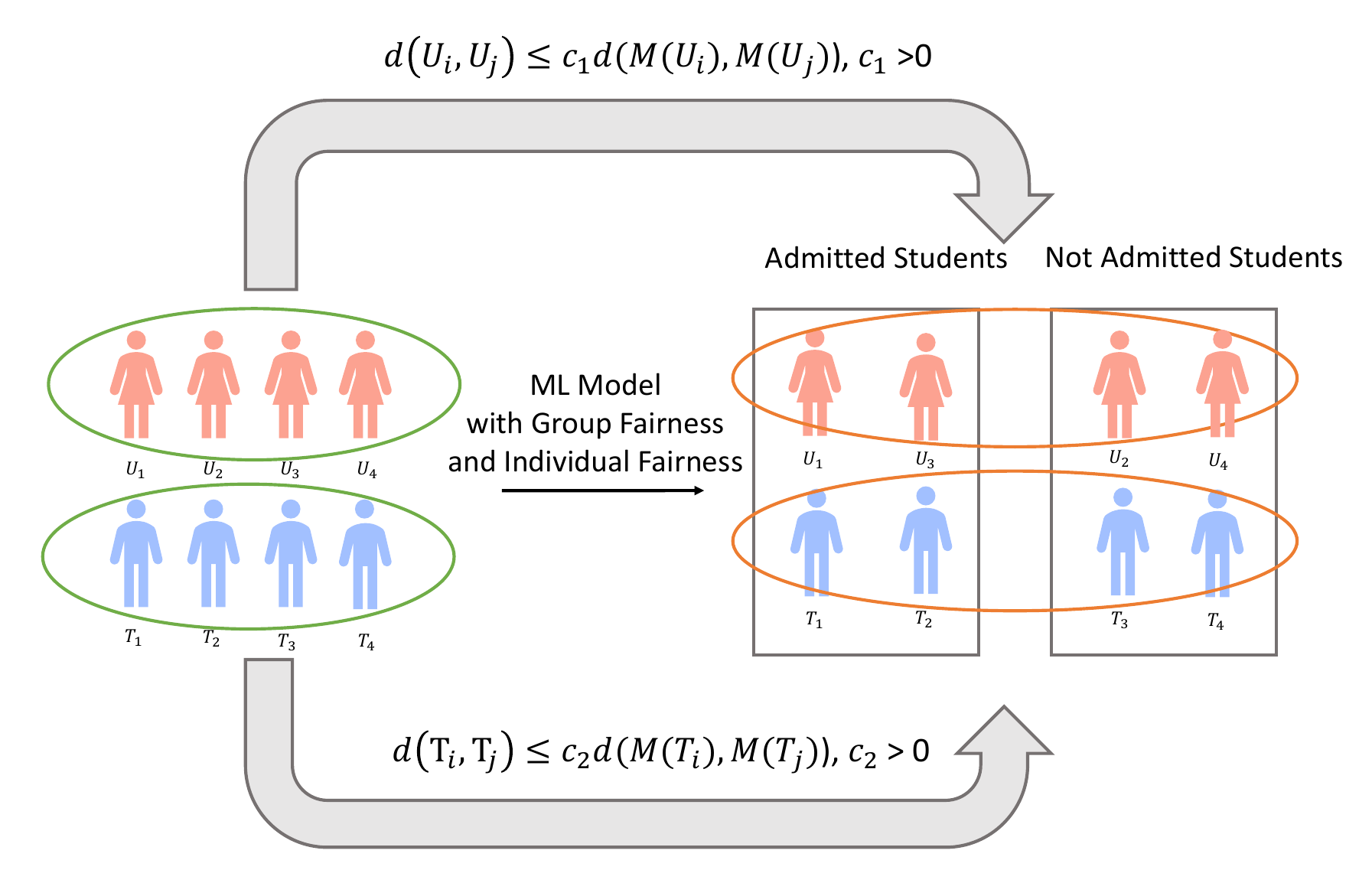} }}%
    \caption{A toy example for student admission model with gender as the sensitive attribute. The red color represents female students, which is also the protected group. The blue color denotes male students. Fig. \ref{fig1} (a) shows the machine learning model can only guarantee the group fairness. Fig. \ref{fig1} (b) shows the machine learning model can guarantee both group and individual fairness.}%
    \label{fig1}%
    \vspace{-5mm}
\end{figure}

To address the abovementioned problems, we design an innovative definition to quantify individual fairness within groups. Further, we develop a novel framework named Fairness for Group and Individual (FairGI), designed to address concerns related to group and individual fairness. Our goal is to address 
two major challenges: (1) how to resolve the conflicts between group fairness and individual fairness and (2) how to ensure both group fairness and individual fairness within groups in graph learning. For the first challenge, we define a new definition of individual fairness within groups. This has been designed to prevent discrepancies between group fairness and individual fairness across different groups.  For the second challenge, we develop a framework FairGI to simultaneously achieve group fairness and individual fairness within groups and maintain comparable accuracy for the model prediction. 

The primary contributions of this paper can be summarized as follows: (1) We introduce a novel problem concerning the achievement of both group fairness and individual fairness within groups in graph learning. To the best of our understanding, this is the first study of this unique issue; (2) We propose a new metric to measure individual fairness within groups for graphs; (3) We propose an innovative framework FairGI, to ensure group fairness and individual fairness within groups in graph learning and maintaining comparable model prediction performance; (4) Comprehensive experiments on various real-world datasets demonstrate the effectiveness of our framework in eliminating both group and individual fairness and maintaining comparable prediction performance. Moreover, the experiments show that even though we only constrain individual fairness within groups, our model achieves the best population individual fairness compared to state-of-the-art models.

\section{Related Work}
\label{sec:rw}
\subsection{Fairness in Machine Learning}
With the advances in machine learning, the applications of machine learning models are widely used in our daily life, including financial services \citep{leo2019machine}, hiring decisions \citep{chalfin2016productivity}, precision medicine \citep{maceachern2021machine}, and so on.  Machine learning models can also be applied in sensitive situations and make crucial decisions for different people. However, recent research shows that the current machine learning models may suffer from discrimination \citep{dressel2018accuracy}. Thus, considering fairness in machine learning models is becoming an important topic when we apply the models to make decisions in our daily life.



The algorithms for fairness in machine learning can be divided into three groups: pre-processing, in-processing, and after-processing. The pre-processing methods mainly focus on adjusting the input data to reduce the unfairness before training the model. The data preprocessing techniques include: re-weighting the sensitive groups to avoid discrimination \citep{kamiran2012data} and re-sampling the data distribution \citep{calmon2017optimized}. The in-processing methods integrate fairness constraints directly into the learning algorithm during the training process, such as adversarial debiasing \citep{zhang2018mitigating} and fairness-aware classifier \citep{zafar2017fairness}. Post-processing techniques adjust the output of the machine learning model after training to satisfy fairness constraints, including rejecting option-based classification \citep{kamiran2012decision} and equalized odds post-processing \citep{hardt2016equality}.

\subsection{Fairness in Graph Learning}

GNNs are successfully applied in many areas. However, they have fairness issues because of built-in biases like homophily, uneven distributions, unbalanced labels, and old biases like gender bias. These biases can make disparities worse through unfair predictions. Therefore, ensuring fairness in GNNs is essential for achieving unbiased decision-making and equitable outcomes. Many methods have been suggested to address fairness in learning from graphs, mainly focusing on three types: group fairness, individual fairness, and counterfactual fairness.

Group fairness approaches, such as FairGNN \citep{dai2021say} and FairAC \citep{guo2023fair}, focus on achieving fairness across different groups, ensuring balanced outcomes at the group level. Recent advancements in this category include FairSample \citep{cong2023fairsample}, FairVGNN \citep{wang2022improving}, FatraGNN \citep{li2024graph}, G-Fame \citep{liu2023fair} and NT-FairGNN \citep{dai2022learning}, which all dedicated to maintaining fairness among groups within graph neural networks.

Individual fairness methods, such as InForm \citep{kang2020inform}, PFR \citep{lahoti2019operationalizing}, and IFGNN \citep{xu2023ifgnn}, focus on fairness at the individual level. These approaches ensure that similar entities within the graph are treated equivalently, prioritizing the notion that fairness must be maintained on a person-by-person basis. However, they often overlook the aspects of group fairness. Guide \citep{song2022guide} equalizes the individual fairness level among different groups, but it still focuses on individual fairness and does not consider group fairness.

Counterfactual fairness methods explore hypothetical scenarios where certain attributes or connections are altered to understand their causal impact on fairness. Methods such as GEAR \citep{ma2022learning}, Nifty\citep{agarwal2021towards}, and CAF \citep{guo2023towardsk} aim to achieve model fairness from a causal perspective, ensuring that individual predictions remain fair by contrasting them with both their original scenarios and counterfactual alternatives.


To underline the prevalent limitations of current methods, it is noteworthy that existing approaches tend to focus on a single dimension of fairness, either group fairness, individual fairness, or counterfactual fairness, which restricts their ability to produce truly fair outcomes for both groups and individuals. \cite{dwork2012fairness} emphasize the inherent conflicts between group fairness and individual fairness, illustrating that combining these two concepts of fairness into one model presents significant challenges. 


To address the inherent conflicts between group and individual fairness, we propose a novel metric, "individual fairness within groups," which specifically aims to improve fairness at the individual level within each group. Additionally, we develop a framework named FairGI that optimizes individual fairness within groups by utilizing a node similarity matrix. FairGI also incorporates adversarial learning to achieve group fairness with respect to EO and SP. This innovative framework represents a more refined and comprehensive method for achieving both group and individual fairness in graph learning, encompassing a broader spectrum of fairness considerations.


\section{Preliminary}
\label{sec:pre}
\subsection{Preliminaries for fairness learning in graphs}


\subsubsection{Individual fairness} Individual fairness emphasizes fairness at the individual level, ensuring that individuals with similar inputs are treated consistently and fairly.  We present the definition of individual fairness in graph learning based on the Lipschitz continuity below \citep{kang2020inform}.





\begin{Definition} (Individual Fairness.) Let $\mathcal{G} = (\mathcal{V}, \mathcal{E})$ be a graph with node set $\mathcal{V}$ and edge set $\mathcal{E}$. $f_G$ is the graph learning model. $Z \in R^{n \times n_z}$ is the output matrix of $f_G$, where $n_z$ is the embedding dimension for nodes, and $n = |V|$. $M \in R^{n \times n}$ is the similarity matrix of nodes. The model $f_G$ is  individual fair if its output matrix $Z$  satisfies 
\begin{equation}
L_{If}(Z)  = \frac{\sum_{v_i \in \mathcal{V}}\sum_{v_j \in \mathcal{V}} \Arrowvert \mathbf{z}_i - \mathbf{z}_j\Arrowvert_F^2 M[i, j] }{2} = Tr(Z^TLZ) \leq m\epsilon,
    \label{eq_def2}
\end{equation}
 where $L \in R^{n \times n}$ is the Laplacian matrix of $M$, $\epsilon \in R^+$ is a constant and $m$ is the number of nonzero values in $M$. $\mathbf{z}_i$ is the $i$th row of matrix $Z$ and $M[i,j]$ is the element in the $i$th row and $j$th column of matrix $M$. $L_{If}$ can be viewed as the population individual bias of model $f_G$.
 \label{def2}
\end{Definition}

\subsubsection{Group fairness}  In this paper, we consider two key definitions of group fairness, which are Statistical Parity (SP) \citep{dwork2012fairness} and Equal Opportunity (EO) \citep{hardt2016equality}. 

\begin{Definition} (Statistical Parity.) 
    \begin{equation}
        \Delta SP = P(\hat{y} = 1|s = 0) - P(\hat{y} = 1|s = 1),
    \end{equation}
  where $\hat{y}$ is the predicted label, $y$ is the ground truth of the label, and $s$ is the sensitive attribute. 
\end{Definition}

\begin{Definition} (Equal Opportunity.) 
    \begin{equation}
        \Delta EO = P(\hat{y} = 1| y = 1, s = 0) - P(\hat{y} = 1| y = 1, s = 1),
    \end{equation}
      where $\hat{y}$ is the predicted label, $y$ is the ground truth of the label, and $s$ is the sensitive attribute. 
\end{Definition}

\subsection{Problem formulation and notations}

\subsubsection{Notations} Let $\mathcal{G} = (\mathcal{V}, \mathcal{E}, \mathcal{X})$ be an undirected graph, where $\mathcal{V}$ is the set of nodes and $\mathcal{E}$ is the set of edges. We have $|\mathcal{V}| = n$.  Let $X \in R^{n \times d}$ be the input matrix of node features with $d$ dimensions. In this paper, we assume the dataset contains a single sensitive feature characterized by binary values. Let $s_i$ be the sensitive attribute of the $i$th node in $\mathcal{G}$ and $\mathcal{S} = \{s_1, s_2, ..., s_n \}$. Let $y_i$ be the target label of the $i$th node in $\mathcal{G}$. The sensitive attribute divides the nodes into two groups. Without loss of generation, we name the group with $s = 1, s \in \mathcal{S}$ as protected group and $s = 0, s \in \mathcal{S}$ as unprotected group. Still, our methods can be easily extended to sensitive attributes with multiple values. 

In this work, we address the unfairness issues in the node classification task on graphs. Our method ensures both group fairness and individual fairness within groups while preserving comparable accuracy performance. The fair graph learning problem is defined below.


\subsubsection{Problem} \textit{Let $\mathcal{G} = (\mathcal{V}, \mathcal{E}, \mathcal{X})$ be an undirected graph with sensitive attribute $s \in \mathcal{S}$. Denote $X \in R^{n \times d}$ as the matrix of node features. Let $\mathcal{P}$ be the set of groups of nodes in $\mathcal{G}$ divided by $\mathcal{S}$, i.e. $\mathcal{V}= \mathcal{V}_{p_2} \bigcup \mathcal{V}_{p_2}$, where $\mathcal{V}_{p_i} = \{v_i| v_i \in \mathcal{V}, s_i = 0, s_i \in \mathcal{S}\}$ and $\mathcal{V}_{p_2} = \{v_i| v_i \in V, s_i = 1, s_i \in \mathcal{S}\}$. A function $f_G: G \to R^{n\times d_h}$ learns the node embeddings of $G$, i.e.,
\begin{equation}
    f_G(\mathcal{G}, \mathcal{S}) = H,
\end{equation}
where $H \in R^{n\times d_h}$ is the node embedding matrix, $|V| = n$ and $d_h$ is the dimension of node embeddings. $f$ satisfies group fairness and individual fairness within groups if and only if $H$ does not contain sensitive information and for $\forall p_k \in \mathcal{P}$,
    \begin{equation}
        \Arrowvert\mathbf{h}_i - \mathbf{h}_j\Arrowvert_2 \leq c_k \Arrowvert \mathbf{x}_i-\mathbf{x}_j\Arrowvert_2, ~\forall i, j \in \mathcal{V}_{p_k}, 
    \end{equation}
where $\mathbf{h}_i$ is $i$th node embedding learned from $f$,  $\mathbf{x}_i$ is the $i$th node feature from $X$ and $c_k \in R^+$ is the Lipschitz constant for group $p_k$. }

\section{Methodology}
\label{sec:method}
\subsection{Framework}

We propose a  novel framework that balances group fairness and individual fairness in the groups to address the problem shown in figure \ref{fig1} and provide a more precise measurement for group fairness. Both group fairness and individual fairness have limitations in real-world application. While group fairness focuses on fairness at the group level, it ignores individual fairness within these groups. Current individual fairness approaches measure individual fairness by Lipchitz Continous \citep{kang2020inform}. However, strictly adhering to individual fairness may lead to conflicts with group fairness \citep{dwork2012fairness}. Our method provides a novel framework to address the above challenges by proposing a novel definition of individual fairness within groups. Combining this new definition with group fairness, we develop a more precise and reliable approach to guarantee fairness in graph learning. 

The detailed algorithm of FairGI is shown in Algorithm \ref{alg1}. 
Since our method allows for model training even when sensitive labels are partly missing, we initially train a sensitive attribute 
estimator utilized GCN \citep{kipf2016semi} to predict the unlabeled 
sensitive attributes. For the node classification task, we employ GAT \citep{wang2019kgat} to generalize node embedding and predict 
target labels. We design a novel loss function to achieve individual fairness within groups in our framework. To guarantee group fairness, we employ an adversarial learning layer that hinders adversaries from precisely predicting sensitive attributes, thereby reducing the bias from sensitive information. Unlike FairGNN \citep{dai2021say}, which solely optimizes SP, we theoretically demonstrate that our adversarial loss function can enhance group fairness in terms of both EO and SP. In addition, we devise a conditional covariance constraint loss function to increase the stability of the adversarial learning process and prove that optimizing the loss function leads to the minimum value of EO.

\begin{figure}[t]
\centering
\vspace{-5mm}
\includegraphics[width=0.7\linewidth]{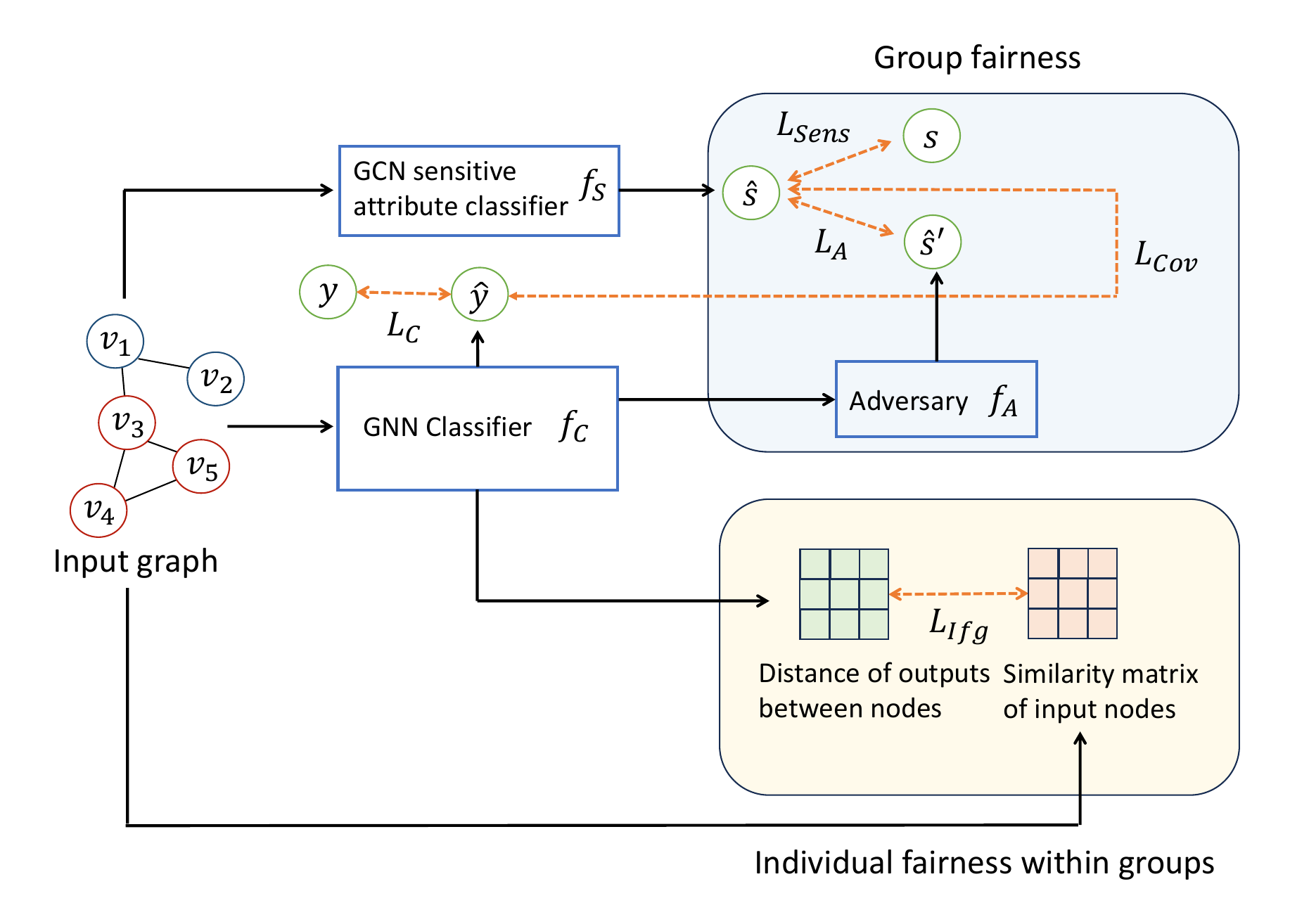}
\caption{Overview of FairGI. Our method comprises three main parts, i.e., an individual fairness module, an group fairness module and a GNN classifier for node prediction.}
\label{framework2}
\vspace{-3mm}
\end{figure}

Fig. \ref{framework2} shows the framework of FairGI. In this framework, we employ a GNN classifier for the node classification task. To ensure individual fairness within groups, we propose a loss function $L_{Ifg}$ to minimize bias among individuals within the same group. To mitigate group bias, we incorporate adversarial learning and develop covariance loss functions to optimize EO and SP. Additionally, we enhance the adversarial learning process by integrating a sensitive attribute classifier $f_S$.


The  comprehensive  loss function of our approach is:
\begin{equation}
    L = L_C +  L_{G}  + \alpha L_{Ifg},
\end{equation}
where $L_C$ is the loss for node label prediction, $L_G$ represents the loss for debiasing group unfairness, and $L_{Ifg}$ represents the loss for mitigating individual unfairness within groups.  


\begin{algorithm}[t]
	\caption{Algorithm of our framework}
 		\label{alg1}
	\begin{algorithmic}[1]
            \State \textbf{Input: $\mathcal{G}(\mathcal{V}, \mathcal{E}), X, \mathcal{S}$}
            \State \textbf{Output: Sensitive attribute classifier $f_S$,  node classifier $f_C$, node label prediction $\hat{y}$}
            \State Train sensitive attribute classifier $f_S$ by given sensitive attribute labels using loss function $L_{Sens}$ in Eq.(\ref{loss_sens}).
            \Repeat
            \State Obtain estimated sensitive attribute $\hat{s}, \hat{s} \in \mathcal{\Tilde{S}}$ by $f_S$.
            \State Optimize $f_G$ to predict the node label by loss function $L_C$ in Eq.(\ref{lc})
            \State Optimize $f_G$ to debias group unfairness by loss function $L_G$ in Eq.(\ref{lg})
            \State Optimize $f_G$ to debias individual unfairness within groups by loss function $L_{Ifg}$ in Eq.(\ref{loss3})
            \State Optimize adversary $f_A$ by $L_A$ in Eq.(\ref{la})
            \Until{Converge}\\
            \Return {$f_S$, $f_G$ and $\hat{y}$}
	\end{algorithmic}
\end{algorithm}




\subsection{Optimization of Individual Fairness within Groups}
\subsubsection{Challenges of balancing individual fairness and group fairness}
We can observe that the definitions of group fairness, especially SP and EO, may contradict individual fairness in specific circumstances.  As noted by \citep{dwork2012fairness},  potential discrepancies can arise between group and individual fairness when the distance between groups is significant. Assume $\Gamma$ represents the protected group and $\Gamma'$ denotes the unprotected group. If there is a considerable distance between individuals in $\Gamma$ and those in $\Gamma'$, strictly adhering to individual fairness may not guarantee similar outcomes for both groups. Thus, this can lead to a conflict with the goal of group fairness, which seeks to maintain equal treatment and opportunity for all groups. Inspired by \citep{dwork2012fairness},  we alleviate the conflicts between SP and individual fairness by loosening the Lipschitz restriction between different groups. 

\subsubsection{Loss function for individual fairness within groups}
To balance group and individual fairness and address the limitation of group fairness, we introduce a novel definition of individual fairness within groups. Based on the proposed definition, we design a loss function that ensures individual fairness within the known groups. Note that our approach is also applicable to groups that are not mutually exclusive.

\begin{Definition} (Individual Fairness within Groups.)
The measurement of individual fairness within group $p$ is:
\begin{equation}
    L_{p}(Z) = \frac{\sum_{v_i \in \mathcal{V}_p}\sum_{v_j \in \mathcal{V}_p} \Arrowvert \mathbf{z}_i - \mathbf{z}_j\Arrowvert_2^2 M[i, j] }{n_p},
    \label{lp}
\end{equation}
where $\mathcal{V}_p$ represents the set of nodes in group $p$ and  $Z$ denotes the node embedding matrix. $\mathbf{z}_i$ is the $i$th row of $Z$. $M$ is the similarity matrix of nodes. $L$ is the laplacian matrix of $M$ and $n_p$ is the number of pairwise nodes in group $p$ with nonzero similarities in $M$. 
\end{Definition}

Our objective focuses on optimizing individual fairness within each group, guaranteeing equitable treatment for individuals in the same groups. To achieve this, our loss function is designed to minimize the maximum individual unfairness over all groups in our loss function.

Firstly, we consider the loss function that  minimizes the maximum unfairness over all the groups as follows:
\begin{equation}
    f^* = \textrm{argmin}_{f \in \mathcal{H}}\{max_{p \in \mathcal{P}}L_p(Z)\},
    \label{loss1}
\end{equation}
where $Z = f(\cdot)$, $\mathcal{H}$ is a class of graph learning models, $\mathcal{P}$ is the set of groups and $L_p$ is the individual loss for group $p$. Motivated by the concept of guaranteeing the optimal situation for the most disadvantaged group, as presented in \citep{diana2021minimax}, we employ the minimax loss function from Eq. (\ref{loss1}). This approach prioritizes minimizing the maximum unfairness across groups, rather than simply aggregating individual fairness within each group to ensure a more equitable outcome.


The optimal solution in Eq.(\ref{loss1}) is hard to obtain, thus we can relax the loss function as expressed in Eq. (\ref{loss2}).  Given an error bound $\gamma$ for each group, the extension of the minimax problem can be formulated as follows:
\begin{equation}
\begin{aligned}
    &\textrm{minmize}_{f \in \mathcal{H}} \sum_{p \in \mathcal{P}} L_p(Z),~
    &\textrm{subject to } L_p(f) \leq \gamma, p \in \mathcal{P}.
\end{aligned}
\label{loss2}
\end{equation}

In our framework, we not only focus on the loss of individual fairness within groups in Eq.(\ref{loss2}), but also consider the loss of label prediction and group fairness. Thus, we further convert Eq.(\ref{loss2}) into the unconstrained loss function shown in Eq.(\ref{loss3}) by introducing Lagrange multiplier $\lambda_p$ to the loss function. We can achieve individual fairness within groups by minimizing the loss function below:
\begin{equation}
     L_{Ifg} = \sum_{p \in \mathcal{P}} L_p(Z) + \sum_{p \in \mathcal{P}} \lambda_p(L_p(Z) - \gamma),
\label{loss3}
\end{equation}
where $\lambda_p$ and $\gamma$ are hyperparameters in our model.

\subsection{Enhancing Group Fairness Through Ensuring Equal Opportunity and Statistical Parity}


In this section, we enhance group fairness by introducing a novel loss function tailored to optimize EO.  This approach marks a significant difference from the covariance constraint proposed by FairGNN \citep{dai2021say}, which only optimizes SP. Additionally, our methodology integrates the loss functions for EO and SP, enabling the simultaneous optimization of both EO and SP, thus presenting a comprehensive framework for enhancing group fairness in model predictions.


\subsubsection{Adverserial learning} As we address the circumstance where certain sensitive labels are absent, we utilized GCN \citep{kipf2016semi} to train the sensitive estimator $f_S$, and the loss function for the sensitive label prediction is:
\begin{equation}
     L_{Sens} = -\frac{1}{|\mathcal{V}|} \sum_{i \in \mathcal{V}} ((s_i)\textrm{log}(\hat{s}_i) + (1 - s_i)\textrm{log}(1 - \hat{s}_i)),
     \label{loss_sens}
\end{equation}
where $s_i$ is the sensitive attribute for the $i$th node, $\hat{s}_i$ is the predicted senstive labels.




To optimize SP, the min-max loss function of adversarial learning  is \citep{dai2021say}:
\begin{equation}
    \min \limits_{\Theta_C} \max \limits_{\Theta_{A}} L_{A_1} = \E_{h \sim p(h|\hat{s} = 1)}[\textrm{log}(f_{A}(h))] + \E_{h \sim p(h|\hat{s} = 0)}[\textrm{log}(1 - f_{A}(h))],
    \label{adv1}
\end{equation}
where $\Theta_C$ is the parameters for graph classifier $f_C$, $\Theta_{A}$ is the parameters for adversary $f_{A}$ and $h$ is the node presentation of the last layer of GNN classifier $f_C$. $h \in p(h|\hat{s} = 1)$ denotes  sampling nodes  from the protected group within the graph $\mathcal{G}$.


While FairGNN demonstrates that optimizing Eq. (\ref{adv1}) can achieve the minimum SP \citep{dai2021say} in the GNN classifier, it does not guarantee the attainment of the minimum EO. Both EO and SP are significant metrics for group fairness. Optimizing solely for SP can adversely affect the performance of EO, leading to model bias. To address the shortage of Eq. (\ref{adv1}), we propose a novel min-max loss function designed for adversarial learning to achieve the minimum EO in Eq. (\ref{adv2}). 
\begin{equation}
        \min \limits_{\Theta_C} \max \limits_{\Theta_{A}} L_{A_2}  = \E_{h \sim p(h|\hat{s} = 1, y = 1)}[\textrm{log}(f_{A}(h))] + \E_{h \sim p(h|\hat{s} = 0, y = 1)}[\textrm{log}(1 - f_{A}(h))].
    \label{adv2}
\end{equation}

The Theorem \ref{thm1} in the Appendix demonstrates that the optimal solution of Eq. (\ref{adv2}) ensures the GNN classifier satisfies $\Delta EO = 0$, given two easily attainable assumptions. In addition, we can also mitigate sensitive information by letting $f_A$ predict the sensitive attribute closer to a uniform distribution, as inspired by \cite{gong2020jointly}.


To simultaneously optimize EO and SP, we combining Eq.(\ref{adv1}) and Eq.(\ref{adv2}) to obtain the loss function for  adversarial learning  as:
\begin{equation}
    L_A = L_{A_1} + L_{A_2}.
    \label{la}
\end{equation}



\subsubsection{Covariance constraint} The limitation of adversarial debiasing is instability. Similar to adversarial learning, FairGNN focuses on optimizing Statistical Parity (SP) through the covariance constraint loss as described by \citep{dai2021say} but ignores the EO. Eq. (\ref{cov1}) shows the 
\begin{equation}
        L_{R_1} = \lvert Cov(\hat{s}, \hat{y})\lvert = \lvert \E[(\hat{s} - \E[\hat{s}](\hat{y} - \E[\hat{y}])]\lvert,
    \label{cov1}
\end{equation}
Note that $\Delta EO = 0$ is not the prerequisite of  $L_{R_1} = 0$. Eq. (\ref{cov1}) does not enhance the model fairness in terms of EO. 

Thus, we propose a covariance constraint loss function to optimize EO as follows:
\begin{equation}
    L_{R_2} = |Cov(\hat{s}, \hat{y}| y = 1)| = |\E[(\hat{s} - \E[\hat{s}|y = 1](\hat{y} - \E[\hat{y})| y =1]|y = 1]|.
    \label{cov2}
\end{equation}


Theorem \ref{thm2} in the Appendix shows that under the mild assumption, $L_{R2} = 0$ is the prerequisite of $\Delta EO = 0$. Consequently, Eq. (\ref{cov2}) effectively optimizes EO.
We can further enhance group fairness in our model by optimizing EO and SP using Eq. (\ref{cov_final}).
\begin{equation}
    L_{Cov} = L_{R1} + L_{R2}.
    \label{cov_final}
\end{equation}
In conclusion, the loss function that we utilize to mitigate group fairness is:
\begin{equation}
    L_{G} = \beta L_A  +\gamma L_{Cov},
    \label{lg}
\end{equation}
where $\beta$ and $\gamma$ are hyperparameters.

\subsection{Node Prediction}
For the node prediction task, we employ GAT \citep{wang2019kgat} to predict node labels. The loss function for GNN classifier $f_C$ is:
\begin{equation}
    L_C = -\frac{1}{|\mathcal{V}|} \sum_{i \in \mathcal{V}} ((y_i)\textrm{log}(\hat{y_i}) + (1 - y_i)\textrm{log}(1 - \hat{y_i})).
    \label{lc}
\end{equation}

In the conclusion, we highlight the distinct contributions of our work in comparison to existing methodologies. Our research introduces a novel problem statement that aims to address both group fairness and individual fairness. This dual focus sets our work apart from others like FairGNN, which only considers group fairness.  Moreover, when considering group fairness, our approach not only optimizes for SP but also for EO, providing a more comprehensive group fairness optimization. In contrast, FairGNN's loss functions are solely directed at SP. This broader and more comprehensive fairness optimization strategy underscores the innovative contributions of our framework.









\section{Experiments}

\label{sec:exp}

\begin{table*}[t]
\centering
\vspace{-5mm}
\caption{Comparisons of our method and baselines on three datasets. $\uparrow$ denotes the larger value is the better and  $\downarrow$ indicates the smaller value is the better. Best performances are in bold. }
\resizebox{0.9\textwidth}{!}{%
\begin{tabular}{c|c|c|c|c|c|c|c}
		\toprule
		Dataset &Method & Acc $\uparrow$ & AUC $\uparrow$ & $\Delta SP$ $\downarrow$ & $\Delta EO$ $\downarrow$ & MaxIG $\downarrow$ & IF $\downarrow$\\
		\hline  
                    &GCN & 68.82$\pm$0.17& 73.98$\pm$0.07 & 2.21$\pm$0.61 & 3.17$\pm$1.10 & 5.69$\pm$0.08 & 899.54$\pm$13.10 \\
                    &GAT & 69.14$\pm$0.68& 74.24$\pm$0.90 & 1.40$\pm$0.64 & 2.86$\pm$0.49 & 6.10$\pm$0.62 & 880.89$\pm$89.97 \\ \cmidrule{2-8}
                    &PRF & 55.39$\pm$0.08& 53.83$\pm$0.02 & 1.08$\pm$0.09 & 1.82$\pm$0.18 & 0.64$\pm$0.01 & 101.26$\pm$1.27\\
                    &InFoRM & 68.77$\pm$0.39& 73.69$\pm$0.10 & 1.84$\pm$0.69 & 3.58$\pm$1.15 & 1.52$\pm$0.05 & 238.41$\pm$7.97\\
            Pokec-n &NIFTY & 65.97$\pm$0.57& 69.87$\pm$0.64 & 4.62$\pm$0.52 & 7.32$\pm$0.94 & 1.87$\pm$0.15 & 310.84$\pm$26.25\\
                    & GUIDE  & 69.46$\pm$0.04 & 74.67$\pm$0.01 & 2.95$\pm$0.11 & 0.80$\pm$0.18 & 0.61$\pm$0.00 & 101.77$\pm$0.28\\
	       & FairGNN  & \textbf{69.86$\pm$0.30} & \textbf{75.58$\pm$0.52} & 0.87$\pm$0.38 & 2.00$\pm$1.08 & 1.26$\pm$0.96&192.29$\pm$142.06 \\
	                & Ours  & 68.86$\pm$0.58 & 75.07$\pm$0.1 & \textbf{0.63$\pm$0.37} & \textbf{0.75$\pm$0.30} & \textbf{0.47$\pm$0.09} & \textbf{67.41$\pm$13.68}\\
	    \midrule
                    &GCN & 69.08$\pm$2.02  & 74.20$\pm$1.69 & 17.12$\pm$7.10 & 10.03$\pm$4.92 & 25.99$\pm$2.12 &17.87$\pm$1.74\\
                    &GAT & 70.80$\pm$3.70  & 72.48$\pm$4.32 & 11.90$\pm$8.94 & 16.70$\pm$10.57& 21.14$\pm$10.86 &20.79$\pm$9.95\\ \cmidrule{2-8}
                    &PRF &55.58$\pm$0.93  & 58.26$\pm$4.45 & 1.99$\pm$0.99 & 2.22$\pm$1.65& 4.47$\pm$2.25 & 3.06$\pm$1.55\\
                    &InFoRM & 68.71$\pm$2.78  & 74.19$\pm$1.85 & 16.64$\pm$5.64 & 12.75$\pm$6.80& 26.52$\pm$8.25 & 18.82$\pm$5.17\\
           NBA   &NIFTY & 70.55$\pm$2.30  & 76.18$\pm$0.83 & 11.82$\pm$4.28 & 5.69$\pm$3.48& 17.14$\pm$5.01 & 11.93$\pm$3.43\\
	                & GUIDE  & 63.31$\pm$2.86 & 67.46$\pm$3.44& 13.89$\pm$5.11 & 10.50$\pm$4.76 & 29.54$\pm$16.34 & 19.84$\pm$10.75    \\
	              & FairGNN & 72.95$\pm$2.10 & 77.37$\pm$1.11 & 1.19$\pm$0.43 & \textbf{0.62$\pm$0.43} & 10.91$\pm$12.59 &18.51$\pm$23.72\\ 
	                & Ours  & \textbf{73.13$\pm$1.75} & \textbf{79.28$\pm$0.46} &\textbf{ 0.43$\pm$0.28} & \textbf{0.62$\pm$0.32} & \textbf{0.12$\pm$0.11} &\textbf{0.08$\pm$0.07}\\

	 \midrule
                    &GCN & 70.35$\pm$0.99& 65.18$\pm$6.58 & 14.55$\pm$6.13 & 13.92$\pm$6.00 & 5.18$\pm$1.34 & 39.11$\pm$6.69\\
                    &GAT & 70.89$\pm$1.84& 71.30$\pm$1.64 & 15.95$\pm$2.40 & 15.96$\pm$2.77 & 5.88$\pm$3.34 & 35.28$\pm$15.41\\ \cmidrule{2-8}
                    &PRF &69.87$\pm$0.09& 69.90$\pm$0.04 & 14.63$\pm$0.78 & 13.96$\pm$0.79 & 5.80$\pm$0.08 & 39.79$\pm$0.63\\
                    &InFoRM & 69.91$\pm$3.70& 65.55$\pm$5.6 & 14.80$\pm$3.84 & 14.82$\pm$4.18 & 4.09$\pm$1.68 & 33.62$\pm$13.86\\
           Credit   &NIFTY & 68.74$\pm$2.34& 68.84$\pm$0.41 & 9.91$\pm$0.30 & 9.07$\pm$0.63 & 2.92$\pm$1.18 & 24.73$\pm$9.75\\
	                & GUIDE  & 62.01$\pm$0.01 & 67.44$\pm$0.01& 13.88$\pm$0.10 & 13.54$\pm$0.06 & \textbf{0.22$\pm$0.01} & 1.90$\pm$0.01    \\
	              & FairGNN &73.40$\pm$0.15  & \textbf{70.18$\pm$0.03} & 3.91$\pm$0.11 & 3.49$\pm$0.25 & 1.88$\pm$0.09 &13.84$\pm$0.70\\ 
	                & Ours  &\textbf{74.09$\pm$0.13}  &  68.81$\pm$0.11&\textbf{ 3.84$\pm$0.22} & \textbf{2.60$\pm$0.20} & \textbf{0.22$\pm$0.01} &\textbf{1.84$\pm$0.10}\\

		 \bottomrule
	\end{tabular}%
}
\label{tab:experiment1}
\vspace{0mm}
\end{table*}

In this section, we conduct a comprehensive comparison between our proposed method and other cutting-edge models, evaluating their performance on real-world datasets to demonstrate the effectiveness of our approach. 

\subsection{Datasets and Baselines}
In this experiment, we utilize three public datasets, Pokec\_n \cite{dai2021say} , NBA \cite{dai2021say}, and Credit \cite{yeh2009comparisons}.



We compare our method with other state-of-art fairness models for graph learning. In our comparison, we include basic GNNs like GCN \citep{kipf2016semi} and GAT \citep{velivckovic2017graph}, which don’t fix bias. We also include GNNs like PFR \citep{lahoti2019operationalizing} and InFoRM \citep{kang2020inform}, aiming at individual fairness. To compare group fairness methods, we include FairGNN \citep{dai2021say}. Additionally, we include NIFTY \citep{agarwal2021towards}, a method based on causal inference for fairness, as a baseline in our study. Further descriptions of datasets and baselines can be found in the Appendix.

\subsubsection{Evaluation metrics}
In this experiment, we primarily focus on analyzing and comparing individual fairness within groups as well as group fairness. Furthermore, we assess the performance of the prediction task by employing metrics such as Area Under the Curve (AUC) and Accuracy (ACC). We include MaxIG, IF, SP, and EO in the experiments for the fairness evaluation metrics. MaxIG is defined as:
\begin{equation}
    {\rm MaxIG} = \max(L_p(Z)), ~p\in \mathcal{P},
\end{equation}
where $L_p(\cdot)$ can be computed in Eq. (\ref{lp}) and $Z$ is the output of GNN Classifier.

\subsection{Results and Analysis}

\subsubsection{Individual unfairness  and group unfairness in graph neural networks}

Based on the experimental results presented in Table \ref{tab:experiment1}, several key findings emerge regarding the performance and biases of various GNNs.

 Traditional GNNs, such as GCN and GAT, exhibit both individual and group biases. This suggests that while these models may have good performance, they do not adequately handle fairness issues. Models that address group fairness, such as FairGNN, demonstrate good performance in group fairness metrics like SP and EO, but struggle with individual fairness metrics, such as IF and MaxIG. This underscores the challenge of simultaneously optimizing for both group and individual fairness.

On the contrary, models like PRF, InFoRM, NIFTY, and GUIDE, which primarily target individual fairness, perform well in mitigating individual biases. However, they have poor performance in group fairness. This dichotomy indicates a potential trade-off regarding group-level fairness while promoting individual fairness. These findings emphasize the need for more comprehensive solutions that simultaneously address individual and group biases.

\subsubsection{Effectiveness of FairGI in mitigating both individual fairness and group fairness }

The results in Table \ref{tab:experiment1} highlight the efficacy of our approach, leading to two primary observations:
(1) Our method outperforms competing methods by ensuring superior group fairness and intra-group individual fairness while retaining comparable prediction accuracy and AUC of the ROC curve; (2) While our technique is constrained only to fairness within groups, it remarkably achieves superior population individual fairness compared to baselines. This suggests that we can attain the pinnacle of population individual fairness by concentrating solely on intra-group individual fairness and overlooking inter-group individual fairness. This outcome is intuitively reasonable given the potential substantial variances among individuals from different groups.


\begin{figure}[t]
\vspace{-5mm}
\minipage{0.32\textwidth}
  \includegraphics[width=\linewidth]{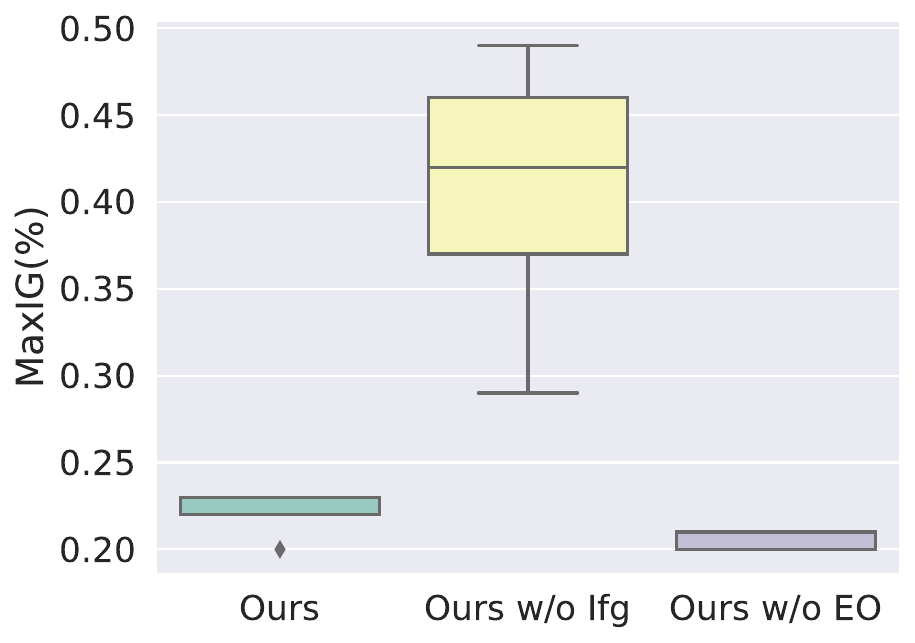}
\endminipage\hfill
\minipage{0.32\textwidth}
  \includegraphics[width=\linewidth]{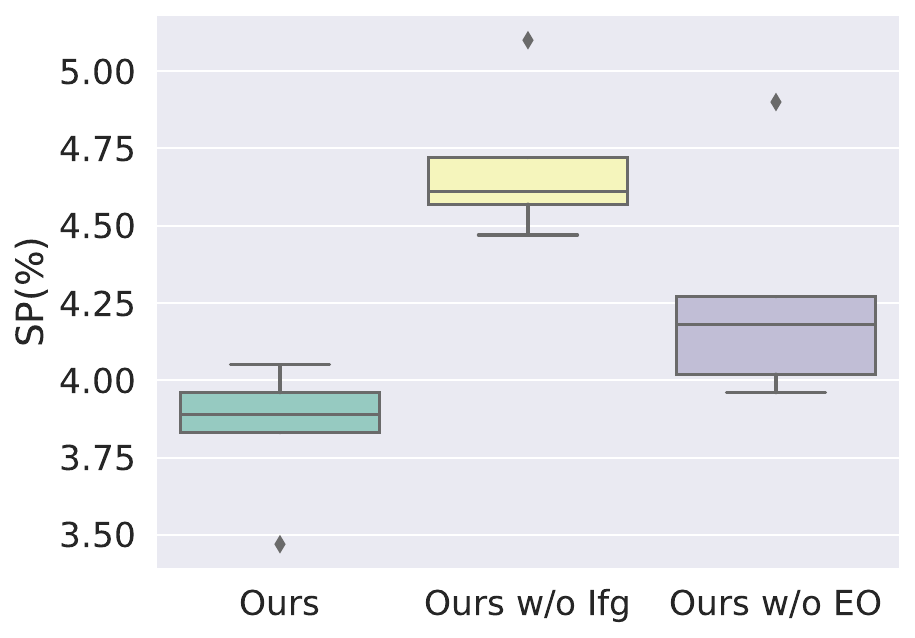}
\endminipage\hfill
\minipage{0.32\textwidth}%
  \includegraphics[width=\linewidth]{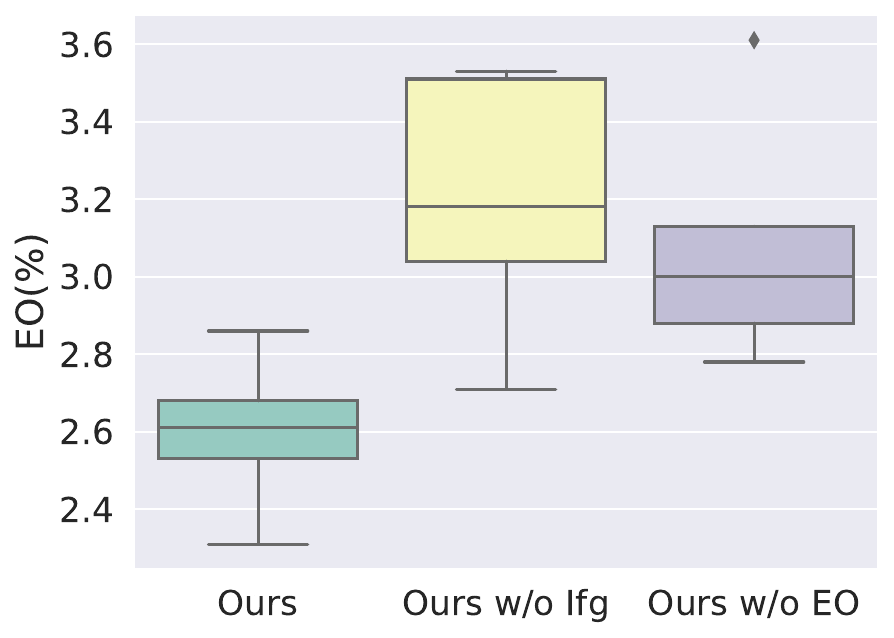}
\endminipage
 \caption{Comparison of our method, our method without loss function of individual fairness within groups, our method without the optimization for EO.}
 \label{ab}
 \vspace{-4mm}
\end{figure}

\subsection{Ablation Studies}

In the ablation study, we examine the impact of two modules, individual fairness within groups and optimization of equal opportunity, on the performance of our method. We conduct a comparison between our method and two of its variants using the Credit dataset. The first variant, "ours w/o Ifg," omits the individual fairness within groups loss $L_{Ifg}$ from our method. The second variant, "ours w/o EO," eliminates the optimizations for equal opportunity, specifically the loss functions $L_{A_2}$ and $L_{R2}$, from our method. 

Figure \ref{ab} illustrates the comparative results. We can observe that upon the removal of $L_{Ifg}$, there is a noticeable increase in MaxIG, EO, and SP, with MaxIG experiencing the most significant rise. This strongly attests to the efficacy of the loss function $L_{Ifg}$ in enhancing individual fairness. When we disregard the optimizations for EO, MaxIG remains relatively unchanged while both EO and SP increase. This highlights the crucial role of $L_{A2}$ and $L_{Cov2}$ in optimizing group fairness. 




\section{Conclusion}
\label{sec:con}
In this paper, we present an innovative problem that considers both group fairness and individual fairness within groups.  In this particular context, we propose a novel definition named MaxIG for individual fairness within groups. Furthermore, we propose a novel framework named FairGI to achieve both group fairness and individual fairness within groups in graph learning. FairGI leverages the similarity matrix to mitigate individual unfairness within groups. Additionally, it exploits the principles of adversarial learning to mitigate group unfairness. Extensive experiments demonstrate that FairGI achieves the best results in fairness and maintains comparable prediction performance. 

\bibliography{arxiv}
\bibliographystyle{arxiv}
\newpage
\appendix
\section{Appendix}
\subsection{Proof of Theorems}
\begin{proposition} Let Eq.(\ref{adv2}) be the loss function of adversary learning. The optimal solution of Eq.(\ref{adv2}) is achieved if and only if $p(h|\hat{s} = 0, y = 1) = p(h|\hat{s} = 1, y = 1)$.
\label{pos1}
\end{proposition}

\begin{proof}
By Proposition 1. in \citep{dai2021say}, the optimal value of adversary is in Eq.(\ref{prop_eq_1})
 \begin{equation}
  f_{A_2}^*(h) = \frac{p(h|\hat{s}=1, y = 1)}{p(h|\hat{s}=1, y = 1)+p(h|\hat{s}=0, y = 1)}.
  \label{prop_eq_1}
 \end{equation}
 We denote $B = p(h|\hat{s} = 1, y = 1)$ and $C =  p(h|\hat{s} = 0, y = 1)$. Thus, the min-max loss function in Eq.(\ref{adv2}) can be written as the following with the optimal solution of adversary:
 \begin{equation}
 \begin{aligned}
     L_{A_2} & = \E_{h\in B}[log\frac{B}{B+C}] + \E_{h\in C}[log\frac{C}{B + C}] \\
     & = \E_{h\in B}[log\frac{B}{\frac{1}{2}(B+C)}] + \E_{h\in C}[log\frac{C}{\frac{1}{2}(B + C)}] - 2log2\\
     & = D_{KL}(B||B+C) + D_{KL}(C||B+C) - 2log2\\
     & = 2 JSD(B||C) - 2log2,
\end{aligned}
 \end{equation}
 where $D_{KL}(\cdot)$ is the Kullback–Leibler divergence and $JSD(\cdot)$ is the Jensen–Shannon divergence.

 We know that $JSD(B||C)$ is non-negative and equals to 0 if and only if distributions $B$ and $C$ are equal. Thus, the loss function $L_{A_2}$ achieves the minimum value if and only if  $p(h|\hat{s} = 0, y = 1) = p(h|\hat{s} = 1, y = 1)$. The proof is adapted to Proposition 4.1 in \cite{dai2021say}.
 

\end{proof}

\begin{Theorem}
Let $\hat{y}$ be the prediction label of GNN classifier $f_G$, $h$ be the node presentation generated by GNN classifier $f_G$. We assume:
\begin{enumerate}
    \item The prediction of sensitive attribute $\hat{s}$ and $h$ are conditionally independent, i.e., $p(\hat{s},h|s, y = 1) = p(\hat{s}|s, y = 1)p(h|s, y = 1)$.
    \item $p(s = 1|\hat{s} = 1, y = 1)  \neq p(s = 1|\hat{s} = 0, y = 1)$.
\end{enumerate}
If Eq.(\ref{adv2}) achieves the global optimum, the prediction of GNN classifier $f_G$ will satisfy equal opportunity, i.e. $p(\hat{y}|s=0, y = 1) = p(\hat{y}|s=1, y = 1)$.
\label{thm1}
\end{Theorem}
Combining Eq.(\ref{adv1}) and Eq.(\ref{adv2}), we have the loss function of  adversarial learning  as:
\begin{equation}
    L_A = L_{A_1} + L_{A_2}.
    \label{la}
\end{equation}

\begin{proof}
    By proposition 7.1, we have $p(h|\hat{s} = 0, y = 1) = p(h|\hat{s} = 1, y = 1)$ when we obtain the optimum solution for the loss function \ref{adv2}. Thus, we have
    \begin{equation}
        \sum_{s \in S} p(h, s|\hat{s} = 1, y = 1) = \sum_{s \in S} p(h, s|\hat{s} = 0, y = 1).
    \end{equation}
    Under the conditionally independent assumption in assumption 1, we have
    \begin{equation}
    \begin{aligned}
         &\sum_{s \in S} p(h|\hat{s} = 1, y = 1)p(s|\hat{s} = 1, y = 1) = \\ &\sum_{s \in S} p(h|\hat{s} = 0, y = 1)p(s|\hat{s} = 0, y = 1).        
    \end{aligned}
         \label{thm1_1}
    \end{equation}

    Reformulating the Eq.(\ref{thm1_1}) and by the assumption 2, we obtain

    \begin{equation}
     \begin{aligned}
        \frac{p(h \lvert s = 1, y = 1)}{p(h \lvert s = 0, y = 1)} 
        & = \frac{p(s = 0\lvert \hat{s} = 0, y =1) - p(s = 0\lvert\hat{s} = 1, y = 1)}{p(s = 1\lvert\hat{s} = 1, y = 1) - p(s = 1\lvert \hat{s} = 0, y = 1)} \\
        & = \frac{ 1 - p(s = 1\lvert \hat{s} = 0, y =1) - 1 + p(s = 1\lvert\hat{s} = 1, y = 1)}{p(s = 1\lvert\hat{s} = 1, y = 1) - p(s = 1\lvert \hat{s} = 0, y = 1)} \\
        & = 1
        \label{thm_2}
    \end{aligned}
    \end{equation}

Thus, we have $p(h|s = 1, y = 1) = p(h|s = 0, y = 1)$, which leads to  $p(\hat{y}|s = 1, y = 1) = p(\hat{y}|s = 0, y = 1)$. The equal opportunity is satisfied when we achieve the global minimum in Eq.(\ref{adv2}).    The proof is adapted to Theorem 4.2 in \cite{dai2021say}.
    
\end{proof}

\begin{Theorem}
    Suppose $p(\hat{s}, h\lvert s, y = 1) = p(\hat{s} \lvert s, y = 1)p(h \lvert s,y = 1)$, when $f_G$ satisfy equal opportunity, i.e. $p(\hat{y}, s \lvert y = 1) = p(\hat{y}\lvert y = 1)p( s \lvert y = 1) $, we have $L_{R_2} = 0$.
    \label{thm2}
\end{Theorem}

\begin{proof}
    Since  $p(\hat{s}, h\lvert y = 1) = p(\hat{s}\lvert y = 1)p(h\lvert y = 1)$, we have the following equation:
\begin{equation}
     \begin{aligned}
    p(h\lvert s, \hat{s}, y = 1) &= \frac{p(h,s, \hat{s}\lvert y = 1)}{p(s,\hat{s} \lvert y = 1)} \\
    & = \frac{p(\hat{s}, h \lvert s, y = 1)}{p(\hat{s}\lvert s, y = 1)}\\ 
    & = p(h\lvert s , y = 1),
    \end{aligned}
\end{equation}
thus, we have $p(\hat{y}\lvert s, \hat{s}, y = 1) = p(\hat{y}\lvert s , y = 1)$.

If $p(\hat{y}, s \lvert y = 1) = p(\hat{y}\lvert y = 1)p( s \lvert y = 1) $, $p(\hat{y}, \hat{s}|y=1)$ can be written as:
        \begin{equation}
     \begin{aligned}
        p(\hat{y}, \hat{s}| y = 1)
        & = \sum \limits_{s \in S} p(\hat{y} \lvert s, y = 1)p(\hat{s}, s \lvert y = 1)\\
        & = p(\hat{y} \lvert y = 1)p(\hat{s} \lvert y = 1).
        \label{thm_3}
    \end{aligned}
    \end{equation}
Thus, we have $L_{R_2} = \lvert Cov(\hat{s}, \hat{y} \lvert y = 1) \rvert = 0$, which proof the theorem. The proof is adapted by the proof of Theorem 4.3 in \cite{dai2021say}.
\end{proof}

\subsection{Datasets and Baselines}
\textbf{Datasets.} In the experiments, we utilize three  datasets. Table \ref{tab:stat3} shows the summary of the datasets. Our datasets demonstrate comprehensive coverage of diverse data categories and varied sample sizes. The detailed descriptions of the datasets are presented below:

\begin{itemize}
    \item The PockeC dataset, presented by \cite{takac2012data}, serves as a benchmark dataset derived from Slovakian social networks, facilitating the evaluation and development of various algorithms and models in this context. This comprehensive dataset encompasses various features for each individual within the network, such as gender, age, educational background, geographical region, recreational activities, working areas, etc. \citep{dai2021say} partitioned the dataset into two distinct subsets, Pokec\_n and Pokec\_z, based on the  provinces of the individuals. Each of the two datasets contains two predominant regions within the relevant provinces. In this experiment, we utilize the Pokec\_n dataset and regard the geographical region as the sensitive attribute. In the node classification task, we use  working areas as the target variable for node prediction. 

    \item The NBA dataset, introduced by \cite{dai2021say}, consists of data from 403 professional basketball players in the National Basketball Association (NBA). The dataset includes features such as age, nationality, salary, and other relevant player attributes. In our experiments, nationality is considered a sensitive attribute, while the target label focuses on determining whether a player's salary is above or below the median.

    \item 
    The Credit dataset is introduced by \cite{yeh2009comparisons}, which offers valuable insights into various aspects of consumer behavior. The dataset includes features such as spending habits and credit history, which are essential for understanding the financial patterns of these individuals. The primary objective of this dataset is to facilitate the prediction of credit card default, with age being identified as the sensitive attribute.

\end{itemize}

\begin{table}[t]
\centering
\vspace{2mm}
\begin{tabular}{@{}lccc@{}}
\toprule
Count & Pokec-n & Credit & NBA\\ \midrule
Number of Nodes & 66,569 & 30,000 & 403    \\ \midrule
Number of node attributes & 59  & 13  & 39\\\midrule
Number of Edges & 729,129 & 304,754 & 16,570   \\ \midrule
Senstive attibute & region & age & nationality\\
 \bottomrule
\end{tabular}
\vspace{2mm}
\caption{Basic statistics of datasets.}
\label{tab:stat3}
\end{table}

\textbf{Baselines.}We compare our methods with other state-of-art models in the node classification task.
\begin{itemize}
    \item \textbf{FairGNN}: FairGNN  is a graph neural network (GNN) model introduced by \cite{dai2021say} employs adversarial learning address the challenges of group fairness in graph representation learning. 

    \item \textbf{GUIDE}: GUIDE was proposed by \cite{song2022guide} to ensure group equality informed individual fairness in graph representation learning.

    \item  \textbf{PRF}: Pairwise Fair Representation (PFR) is a graph learning method introduced by \cite{lahoti2019operationalizing} to achieve individual fairness in graph representation learning.
    
    \item \textbf{InFoRM}: Individual
Fairness on Graph Mining (InFoRM) \citep{kang2020inform} achieves individual fairness in graph representation learning by employing Lipschitz continuity.

\item \textbf{NIFTY}: \cite{agarwal2021towards} propose NIFTY (unifying fairness and stability) that applies the Lipschitz condition to achieve counterfactual fairness in graph learning. 

\end{itemize}
\subsection{Experiment Settings}
In this experiment, we compare our method to the state of art models for fairness in graph learning. Here we use Graph Convolutional Network (GCN) \citep{kipf2016semi} and Graph Attention Network (GAT) \citep{wang2019kgat} as the vanilla comparison models since they do not apply fairness skills. We also include graph learning models with group fairness like FairGNN \citep{dai2021say}. Graph learning models with individual fairness include GUIDE \cite{song2022guide}, PRF \citep{lahoti2019operationalizing} and InFoRM \citep{kang2020inform}. Graph learning models with counterfactual fairness such as NIFTY \citep{agarwal2021towards}. The parameters of our method are shown in Table \ref{tab:par3}.
We apply four datasets in the experiments. For each dataset we randomly divide them into training set, test set, and validation set with ratios of 50\%, 25\% and 25\%.

\begin{table}[t]
\centering
\vspace{2mm}
\begin{tabular}{@{}lcccc@{}}
\toprule
Count & Pokec-n & Credit & NBA  \\ \midrule
$\alpha$, coefficient of $L_{Ifg}$ & 1e-9 & 0.5 & 1e-9  \\ \midrule
$\beta$, coefficient of $L_A$ & 0.02 & 0.8 & 0.01 \\ \midrule
$\gamma$   & 0.004 &0.004 & 0.004 \\ \midrule
$\lambda_1$   & 0.5 & 0.5 & 0.5 \\ \midrule
$\lambda_2$   & 1.25 & 1.25& 1 \\ \midrule
$\eta$, coefficient of $L_{Cov}$  & 3 &6 & 16\\\midrule
number of sensitive labels& 200 & 500& 50 \\\midrule
learning rate & 0.0005 & 0.001& 0.001 \\\midrule
weight decay & 1e-5 & 1e-5 & 1e-5 \\

 \bottomrule
\end{tabular}
\vspace{2mm}
\caption{Hyper parameter setting for datasets.}
\label{tab:par3}
\end{table}

\begin{figure}[t]
\vspace{0mm}
\minipage{0.32\textwidth}
  \includegraphics[width=\linewidth]{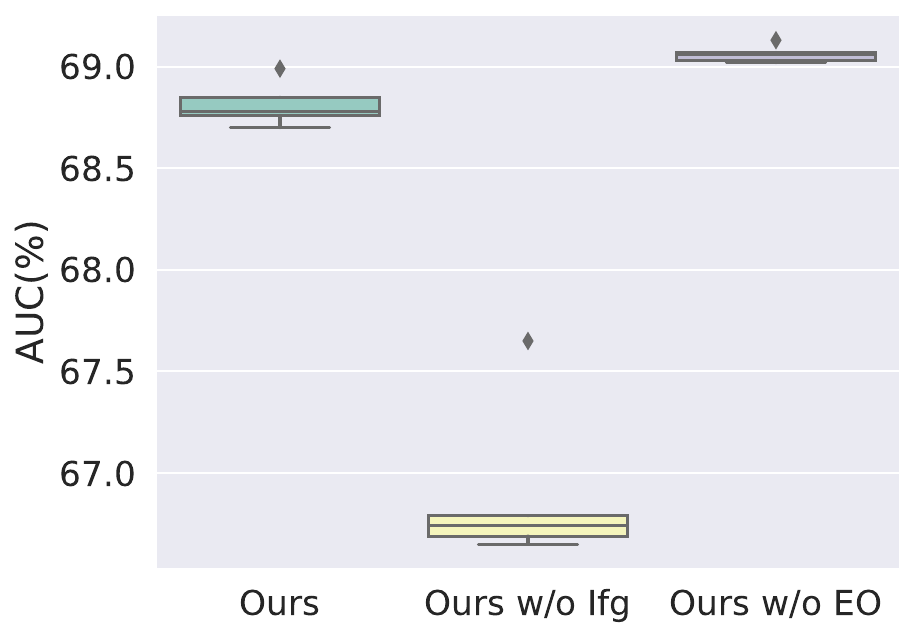}
\endminipage\hfill
\minipage{0.32\textwidth}
  \includegraphics[width=\linewidth]{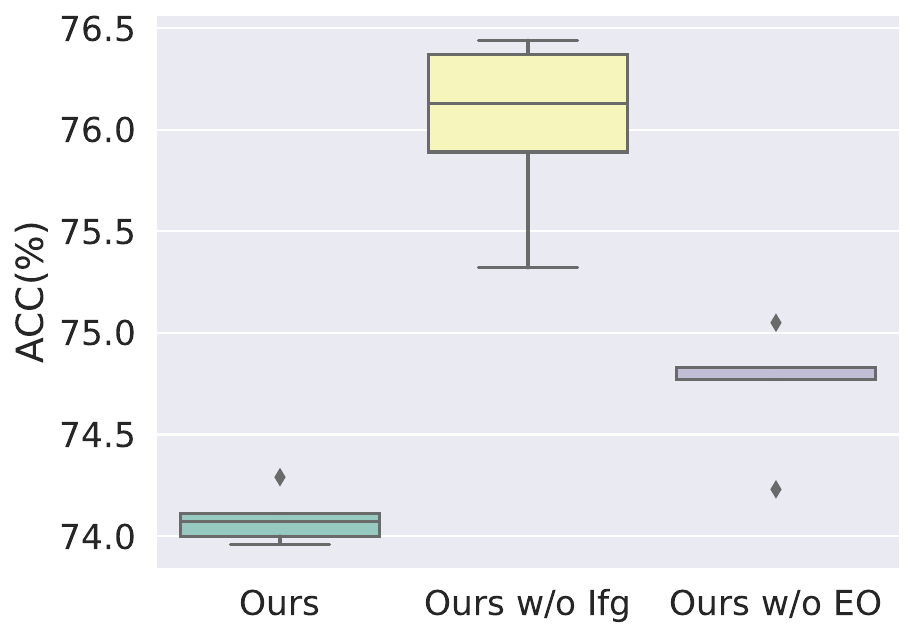}
\endminipage\hfill
\minipage{0.32\textwidth}%
  \includegraphics[width=\linewidth]{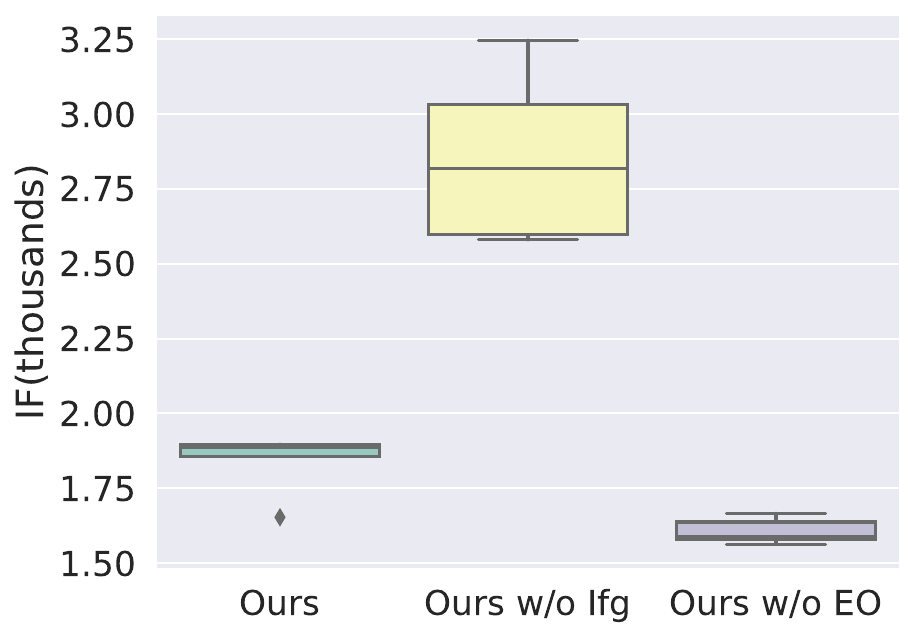}
\endminipage
 \caption{Comparison of our method, our method without loss function of individual fairness within groups, our method without the optimization for EO.}
 \label{ab_app}
 \vspace{-4mm}
\end{figure}
\subsection{Additional Analysis on  Ablation Studies}

Figure \ref{ab_app} presents performance comparisons based on AUC, ACC, and IF. We note that eliminating individual fairness losses within groups results in a marginal increase in ACC compared to our approach. When we exclude EO optimization losses, both ACC and AUC exhibit a non-significant increase, demonstrating that our method can maintain comparable accuracy.

Furthermore, We can observe that if we remove the loss of individual fairness within groups, the performance of IF becomes worse. This observation demonstrates the effectiveness of the loss function of individual fairness within groups, i.e., $L_{Ifg}$.

\end{document}